\documentclass[letterpaper, 10 pt, conference]{ieeeconf}  

\IEEEoverridecommandlockouts                              

\usepackage{amsmath} 
\usepackage{amssymb}  

\usepackage{amsthm}
\usepackage{cite}

\usepackage{algorithm}
\usepackage[noend]{algpseudocode}
\usepackage{color}
\usepackage{graphicx}
\usepackage{booktabs}
\usepackage{xcolor}
\usepackage{siunitx}
\usepackage[hidelinks]{hyperref}

\newtheorem{thm}{Theorem}[section]

\newtheorem{rem}{Remark}
\newtheorem{assume}{Assumption}

\DeclareMathOperator*{\argmin}{argmin}

\title{\LARGE \bf Differentiable Optimization Based Time-Varying \\ Control Barrier Functions for Dynamic Obstacle Avoidance}

\author{Bolun Dai$^{1}$, Rooholla Khorrambakht$^{1}$, Prashanth Krishnamurthy$^{1}$, Farshad Khorrami$^{1}$
\thanks{This work was supported in part by the NYUAD Center for Artificial Intelligence and Robotics (CAIR), funded by Tamkeen under the NYUAD Research Institute Award CG010.}
\thanks{$^{1}$Bolun Dai, Rooholla Khorrambakht, Prashanth Krishnamurthy, and Farshad Khorrami are with Control/Robotics Research Laboratory, Electrical~\&~Computer Engineering Department, Tandon School of Engineering, New York University, Brooklyn, NY 11201
{\tt\footnotesize \{bolundai, rk4342, prashanth.krishnamurthy, khorrami\}@nyu.edu}}%
}

\begin{document}
\maketitle
\thispagestyle{empty}
\pagestyle{empty}
\begin{abstract}
   Control barrier functions (CBFs) provide a simple yet effective way for safe control synthesis. Recently, work has been done using differentiable optimization (diffOpt) based methods to systematically construct CBFs for static obstacle avoidance tasks between geometric shapes. In this work, we extend the application of diffOpt CBFs to perform dynamic obstacle avoidance tasks. We show that by using the time-varying CBF (TVCBF) formulation, we can perform obstacle avoidance for dynamic geometric obstacles. Additionally, we show how to extend the TVCBF constraint to consider measurement noise and actuation limits. To demonstrate the efficacy of our proposed approach, we first compare its performance with a model predictive control based method and a circular CBF based method on a simulated dynamic obstacle avoidance task. Then, we demonstrate the performance of our proposed approach in experimental studies using a 7-degree-of-freedom Franka Research 3 robotic manipulator.
\end{abstract}
\section{Introduction}
With robotic systems being deployed in more complicated and dynamic scenarios, it is crucial to ensure their safe interaction with their surroundings~\cite{DaiKPK23, RawlingsMD17, DaiKPK21}. Model predictive control (MPC) based methods~\cite{RawlingsMD17} are widely used for safety-critical tasks. However, the computation time for MPC-based methods greatly increases when dealing with nonlinear system dynamics and nonconvex safety constraints. Recently, control barrier functions (CBFs)~\cite{AmesCENST19} have gained popularity in ensuring safety in robotic applications. CBF-based methods transform nonlinear and nonconvex safety constraints into linear ones~\cite{AmesXGT17, DaiKK22, DaiHKK23}, which makes them more suitable for complex safe set geometries.

CBFs have been previously used for operational space control~\cite{MurtazaAWH22}, teleoperation~\cite{XuS18}, and kinematic control~\cite{SingletaryKA22}. However, in~\cite{MurtazaAWH22, XuS18, SingletaryKA22}, the CBF is formulated between a point and a geometric shape, which makes it challenging to apply these methods to tasks where both the robot and the obstacles cannot be modeled as points. To enable CBF-based methods to consider the geometry of the robot and its environment, in~\cite{ThirugnanamZS22}, a duality-based convex optimization based approach was used to compute CBFs between polytopes. In~\cite{SingletaryGMSA22}, the CBF is formulated using the signed-distance function (SDF) between two geometric shapes. However, since SDFs are not globally continuously differentiable, \cite{SingletaryGMSA22} used an approximation of the SDF-based CBF, resulting in a conservative constraint that reduces the feasible set. An alternative solution is provided in~\cite{DaiKKGTK23}, where a differentiable optimization (diffOpt) based CBF is used, which recovers the entire safe set. However, the aforementioned works only consider time-invariant safe sets and cannot be directly deployed for problems with time-varying safe sets, such as dynamic obstacle avoidance tasks.

\begin{figure}[t!]
    \centering
    \includegraphics[width=0.45\textwidth]{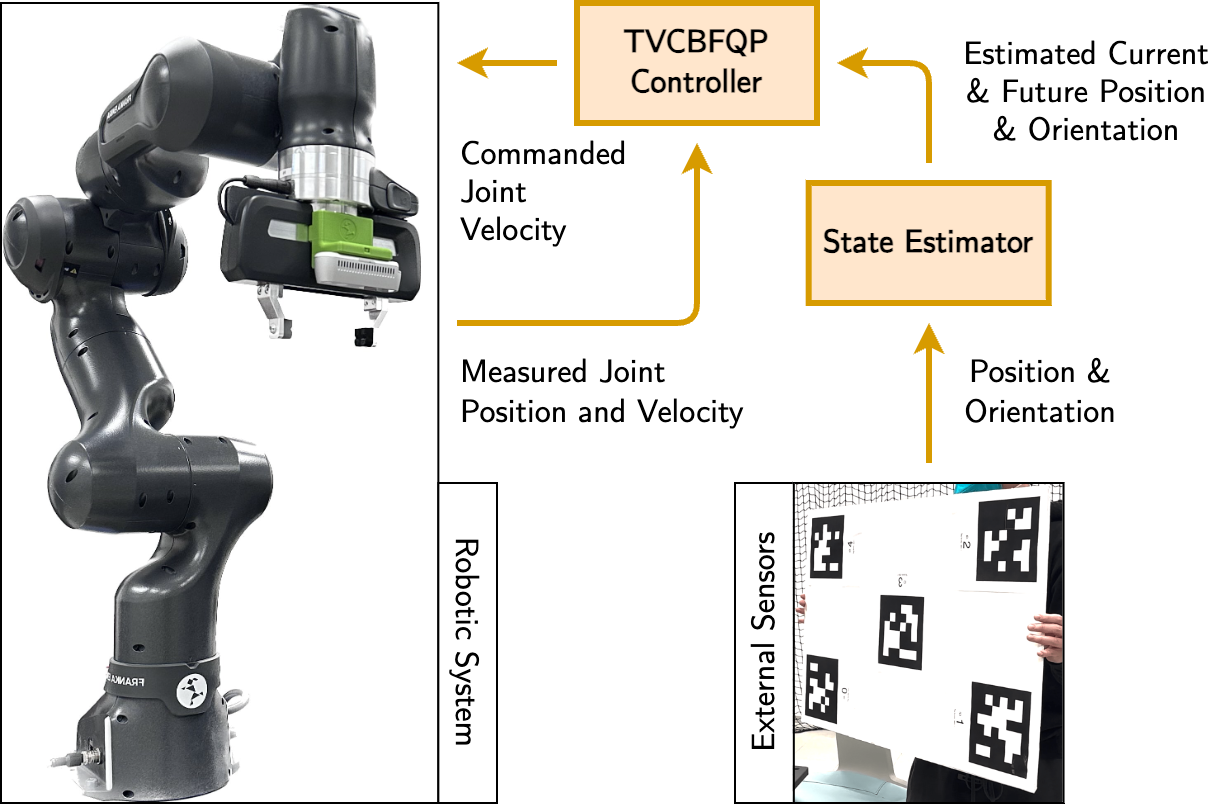}
    \caption{Control pipeline for using time-varying diffOpt CBFs for dynamic obstacle avoidance.}
    \label{fig:tvcbf_open}
\end{figure}

The ability to avoid dynamic obstacles is crucial when deploying robotic systems in the real world. Early works studying dynamic obstacle avoidance tasks proposed using velocity obstacles~\cite{FioriniS98} and dynamic windows~\cite{FoxBT97}. More recently, work has been done using MPC-based methods that utilize adaptively updated half-space constraints~\cite{WhiteJWH22}. However, in addition to the limitations of MPC-based approaches, \cite{WhiteJWH22} models both the robot and the obstacle as spheres, which may generate overly conservative motions when the robot and obstacle geometry consists of shapes that cannot be closely approximated as spheres, e.g., the links of robot manipulators and flat boards. Work has also combined MPC and CBF-based methods for dynamic obstacle avoidance~\cite{JianYLLLWL23}. Similar to~\cite{WhiteJWH22}, the robot and obstacle geometries are limited to spheres and ellipsoids.

This paper extends the diffOpt CBF approach to handle time-varying safe sets and applies the proposed method to dynamic obstacle avoidance tasks. The overall control pipeline is illustrated in Fig.~\ref{fig:tvcbf_open}. Compared to existing methods for dynamic obstacle avoidance, the diffOpt-based formulation~\cite{TracyHM22} in our proposed method enables the consideration of a wider variety of obstacle and robot geometries. The main contribution of this paper is as follows: (1) we extended the diffOpt CBF formulation proposed in~\cite{DaiKKGTK23} to settings with time-varying safe sets and showed that this extension enables the application to dynamic obstacle avoidance tasks; (2) proposed a systematic approach to address noise in dynamic obstacle position measurements through a novel formulation of a ``most unsafe'' configuration within a certain Mahalanobis distance of the measurement distribution (in~\cite{DaiKKGTK23}, the poses of the static obstacles are known beforehand); (3) proposed an approach to address actuation limits by introducing an additional scaling factor $a_v$ computed from the relative velocity and a hyperparameter $b$ that is tuned based on the actuation limits (in~\cite{DaiKKGTK23}, actuation limits are not considered); (4) performed extensive experiments in both simulations and on a seven degree-of-freedom (DOF) FR3 robotic manipulator to demonstrate the efficacy of our proposed method. The remainder of this paper is structured as follows. Section~\ref{sec:preliminaries} briefly reviews CBFs and diffOpt CBFs. In Section~\ref{sec:problem_formulation}, the safe robotic control problem for dynamic obstacle avoidance is formulated. In Section~\ref{sec:method}, we present the time-varying diffOpt CBF formulation and how to account for state estimation uncertainty and actuation limits. In Section~\ref{sec:experiments}, we show the efficacy of our approach by comparing our proposed approach with an MPC-based method and other CBF-based methods on a dynamic obstacle avoidance task with a non-ellipsoid obstacle in simulation. Also, we perform three dynamic obstacle avoidance tasks in the real world using the 7-DOF Franka Research~3 (FR3) robotic arm. Finally, in Section~\ref{sec:conclusion}, we conclude the paper with a discussion on future directions.

\section{Preliminaries}
\label{sec:preliminaries}

In this section, we provide a brief introduction to CBFs and diffOpt CBFs.

\subsection{Control Barrier Function}
Consider a control affine system
\begin{equation}
    \dot{x} = F(x) + G(x)u
\label{eq:control_affine_sys}
\end{equation}
where the state is represented as $x \in \mathbb{R}^n$, the control as $u \in \mathbb{R}^m$, the drift as $F: \mathbb{R}^n \rightarrow \mathbb{R}^n$, and the control matrix as $G: \mathbb{R}^n \rightarrow \mathbb{R}^{n \times m}$. Consider two sets $\mathcal{C}$ and $\mathcal{D}$ with the relationship $\mathcal{C} \subset \mathcal{D} \subset \mathbb{R}^n$. Let a continuously differentiable function $h:\mathcal{D} \rightarrow \mathbb{R}$ have $\mathcal{C}$ as its 0-superlevel set, i.e., $\mathcal{C} = \{x \mid h(x) \geq 0, x\in\mathbb{R}^n\}$ and $\partial h/\partial x \neq 0$ for all $x\in\partial\mathcal{C}$, where $\partial\mathcal{C} \subset \mathbb{R}^n$ represents the boundary of $\mathcal{C}$. Then, if 
\begin{equation}
    \sup_{u\in\mathcal{U}}\Big[\frac{\partial h(x)}{\partial x}\Big(F(x) + G(x)u\Big)\Big] \geq -\Gamma(h(x))
    \label{eq:CBF_constraint}
\end{equation}
holds for all $x\in\mathcal{D}$ with $\Gamma: \mathbb{R}\rightarrow\mathbb{R}$ being an extended class $\mathcal{K}_\infty$ function\footnote{Extended class $\mathcal{K}_\infty$ functions are strictly increasing with $\Gamma(0) = 0$.}, $h$ is a CBF on $\mathcal{C}$. 

\subsection{Differentiable Optimization Based Control Barrier Function Formulation}
In~\cite{DaiKKGTK23}, a diffOpt CBF is constructed using an optimization problem that finds the minimum uniform scaling factor of two convex objects under which they collide~\cite{TracyHM22, GilbertO94, OngG96}
\begin{align}
\label{eq:differentiable_collision}
    \min_{p, \alpha}\ &\ \alpha\\
    \mathrm{subject\ to}\ &\ p \in \mathcal{P}_A(\alpha, \psi_A), p \in \mathcal{P}_B(\alpha, \psi_B), \alpha > 0\nonumber
\end{align}
where $\alpha\in\mathbb{R}_+$ is the uniform scaling factor, $p\in\mathbb{R}^3$ is the point of intersection after scaling the two convex objects with the scaling factor $\alpha$, $\psi_A, \psi_B \in \mathrm{SE}(3)$ represents the position and orientation of $A$ and $B$, respectively, and $\mathcal{P}_A, \mathcal{P}_B: \mathbb{R}_+ \times\mathrm{SE}(3) \rightrightarrows \mathbb{R}^3$ represent the two convex objects after applying a scaling factor of $\alpha$. Define the solution to~\eqref{eq:differentiable_collision} as $\alpha^\star$ and $p^\star$. It can be seen that the two objects are not in collision when $\alpha^\star > 1$, while the objects are either touching or colliding when $\alpha^\star \leq 1$. If $A$ is the robot and $B$ is the obstacle, then the corresponding CBF has the form of
\begin{equation}
    \label{eq:diff_opt_cbf}
    h(x, \psi_B) = \alpha^\star(x, \psi_B) - \beta.
\end{equation}
The $\beta$ value corresponds to an offset to the scaling factor. Given that safety is ensured when $\alpha^\star - 1 > 0$, we set $\beta \geq 1$. The larger $\beta$ is, the more conservative the CBF is. Note that since $\psi_A$ is a function of the state of the robot $x$, we write $h(x, \psi_B)$ instead of $h(\psi_A, \psi_B)$. Additionally, based on the implicit function theorem~\cite{Dini07}, the diffOpt problem in~\eqref{eq:differentiable_collision} computes $\partial\alpha/\partial\psi_A$, which is used to compute the partial derivative of the CBF in~\eqref{eq:CBF_constraint}:
\begin{equation}
    \frac{\partial h}{\partial x}(x, \psi_B) = \frac{\partial\alpha^\star}{\partial \psi_A}(x, \psi_B)\frac{\partial\psi_A}{\partial x}(x).
\end{equation}
\section{Problem Formulation}
\label{sec:problem_formulation}
Let the robot dynamics be defined as~\eqref{eq:control_affine_sys}. For obstacle avoidance tasks with a single obstacle, the safe set $\mathcal{C}$ is defined as the complement set of the interior and surface of the obstacle. For dynamic obstacles, the safe set would be time-varying, which is denoted as $\mathcal{C}(t)$. When there exist $N \in \mathbb{Z}_+$ obstacles, the safe set at time $t$ can be written as
\begin{equation}
    \mathcal{C}(t) = \bigcap_{i=1}^{N}{\mathcal{C}_i(t)}
\end{equation}
where $C_i(t)$ represents the safe set with respect to the $i$-th obstacle at time $t$. Let the robot state at $t = 0$ be $x_0 \in \mathcal{C}(0)$. This work aims to find a systematic approach that ensures the robot does not collide with any of the obstacles, i.e., $x(t) \in \mathcal{C}(t)$ for $t > 0$. 

\begin{rem}
\label{rem:obstacle_motion}
Finding a control sequence that ensures the robot's safety is not always possible. For example, if the robot has a fixed base, no control sequence will be able to make the robot avoid an obstacle that hits the fixed base. Another example is when the robot is significantly less agile than the obstacle. \textbf{This work does not consider such cases. We assume there exists a control sequence within the robot's control authority that guarantees safety}.
\end{rem}
\section{Method}
\label{sec:method}
In this section, we present the proposed method for performing dynamic obstacle avoidance using a diffOpt-based time-varying CBF (TVCBF). 

\subsection{Motivating Example -- Moving Circles}
\begin{figure}[t!]
    \centering
    \includegraphics[width=0.45\textwidth]{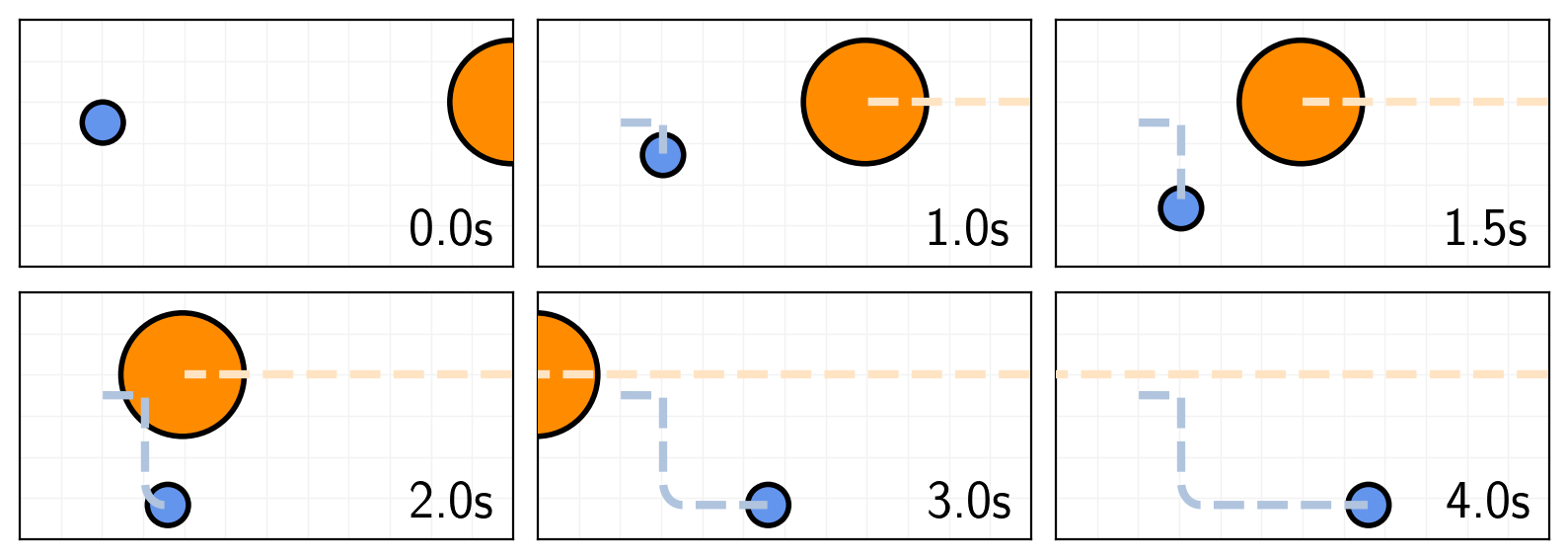}
    \caption{This figure illustrates the motion generated by the TVCBFQP controller. The robot we control is in blue. The obstacle is in orange. The light blue and orange dashed lines represent the path the robot and the obstacle traveled, respectively. The timestamp of each snapshot is given in the lower right corner of each figure. Each small grid is $2~\si{m}\times2~\si{m}$.}
    \label{fig:moving_circle_visualization}
\end{figure}
A motivating example is provided to aid in the presentation of our method, which will be used throughout Section~\ref{sec:method}. The example consists of a circular robot with radius $R_r = 0.5~\si{m}$ and a circular obstacle with radius $R_o = 1.5~\si{m}$. Define the position of the robot and the obstacle as $p_r\in\mathbb{R}^3$ and $p_o\in\mathbb{R}^3$, respectively. Set the initial position of the robot and the obstacle to $p_r^0 = (-5.0, -0.5, 0.0)$ and $p_o^0 = (5.0, 0.0, 0.0)$, respectively. Define the velocity of the obstacle as $v_o \in \mathbb{R}^3$, which is along the negative $x$ direction. The dynamics of the robot is
\begin{equation}
\label{eq:circle_robot_dyn}
    \dot{p}_r = \begin{bmatrix}
        \dot{p}_{r, x} & \dot{p}_{r, y} & \dot{p}_{r, z}
    \end{bmatrix}^T = \begin{bmatrix}
        u_x & u_y & 0
    \end{bmatrix}^T
\end{equation}
with $u \in \mathcal{U} = \{(u_x, u_y) \mid u_x \in \mathbb{R}_+, u_y \in \mathbb{R}\}$ representing the agent's control command. The robot's goal is to move along the positive $x$ direction while avoiding collision with the obstacle. For this moving circles example, the constraints in~\eqref{eq:differentiable_collision} have the form of~\cite{TracyHM22}
\begin{equation}
    \begin{bmatrix}
        0\\
        -r_i
    \end{bmatrix} - \begin{bmatrix}
        \mathbf{0}_{1\times3} & -1\\
        -\mathbf{I}_{3\times3} & \mathbf{0}_{3\times1}
    \end{bmatrix}\begin{bmatrix}
        p\\
        \alpha
    \end{bmatrix} \in \mathcal{Q}_4,\ i = \{A, B\}
\end{equation}
where $r_i \in \mathbb{R}^3$ represents the position of object $i$ and the four-dimensional second-order cone $\mathcal{Q}_4$ is defined as $\mathcal{Q}_4 = \{[t,\ x^T] \mid x \in \mathbb{R}^3, t \in \mathbb{R}, \|x\|_2 \leq t\}$.

\subsection{Time-Varying Control Barrier Function}
\label{sec:method-tvcbf}
In~\cite{DaiKKGTK23}, the safe set $\mathcal{C}$ is assumed to be time-invariant. When the safe set is time-varying, TVCBFs would need to be used. To emphasize that TVCBFs are computed with respect to a time-varying safe set, they are written as $h(x, \psi(t))$, where $\psi(t)$ denotes the configuration of the safe set at time $t$. The TVCBF constraint~\cite{Xu18} has the form of 
\begin{equation}
\label{eq:tvcbf-constraint}
    \frac{\partial h}{\partial x}(x, \psi(t))\dot{x} + \frac{\partial h}{\partial t}(x, \psi(t)) \geq -\gamma h(x, \psi(t)).
\end{equation}
with $\gamma\in\mathbb{R}_+$ corresponding to how fast the CBF value is allowed to change, the larger the $\gamma$ value, the faster the CBF value can change. For robotics tasks, to apply the diffOpt TVCBF, both the robot and the obstacle must be encapsulated by convex primitive shapes. Define the number of primitives assigned to the robot and obstacle as $n_r\in\mathbb{N}$ and $n_o\in\mathbb{N}$, respectively. Then, there will be $n_r \times n_o$ TVCBF constraints. Adopting the CBF-based quadratic program (CBFQP) formulation, we write the time-varying CBFQP (TVCBFQP) as
\begin{align}
\label{eq:tvcbfqp}
    \min_{u\in\mathcal{U}}\ &\ \|u - u_\mathrm{ref}\|\ \mathrm{subject\ to}\ \frac{\partial h_{ij}}{\partial x}\dot{x} + \frac{\partial h_{ij}}{\partial t} \geq -\gamma h_{ij}
\end{align}
where $i = 1, \cdots, n_r$, $j = 1, \cdots, n_o$, $u_\mathrm{ref}\in\mathbb{R}^m$ is the reference control action, $\mathcal{U}\subset\mathbb{R}^m$ is the admissible set of controls (control amplitude limits), and $h_{ij}$'s dependency on $(x_i, \psi_j(t))$ is omitted for brevity. The $\partial h/\partial t$ term is estimated using a finite difference approach with a state estimator. Similarly, $\partial h/\partial t$ can also be computed as $(\partial h / \partial \psi)(\partial \psi / \partial t)$ where $\partial h / \partial \psi$ is obtained from~\eqref{eq:differentiable_collision} and $\partial \psi / \partial t$ is obtained from a state estimator. To show the effectiveness of diffOpt TVCBFQPs, we apply it to solving the Moving Circles example. When $v_o = [-4.0, 0.0, 0.0]^T$ and $u_\mathrm{ref} = [2, 0]^T$, the result is shown in Fig.~\ref{fig:moving_circle_visualization}. It may be noted that the robot avoids collision with the obstacle while continuing its movement along the positive $x$ direction. 
\begin{rem}
For the Moving Circles example, if we replace the diffOpt \textbf{TVCBF} with the diffOpt \textbf{CBF} proposed in~\cite{DaiKKGTK23}, it is observed that the robot will collide with the obstacle.
\end{rem}

\subsection{Measurement Noise}
\label{sec:measurement_noise}
To implement our TVCBFQP controller, we need to predict the motion of the obstacles online. Any sensor providing such measurements will be corrupted by measurement noise. If not considered, the measurement noise may lead to safety violations of the robotic system. We utilize an Extended Kalman Filter (EKF) to provide a probabilistic estimation of the state in the form of a Gaussian distribution $\mathcal{N}(\mu, \Sigma)$ with $\mu\in\mathbb{R}^{n_m}$, $\Sigma\in\mathbb{R}^{n_m \times n_m}$, and $n_m\in\mathbb{N}$ being the dimension of $\mu$. Denote the Mahalanobis distance~\cite{Mahalanobis18} $d_M$ for a point $y$ with respect to to $\mathcal{N}(\mu, \Sigma)$ as
\begin{equation}
    d_M(y, \mu, \Sigma) = \sqrt{(y - \mu)^T\Sigma^{-1}(y - \mu)} \in \mathbb{R}_+
\end{equation}
where $y\in \mathbb{R}^{n_m}$ and $d_M: \mathbb{R}^{n_m} \times \mathbb{R}^{n_m} \times \mathbb{R}^{n_m \times n_m} \rightarrow \mathbb{R}_+$. Denote the probability of sampling a random variable from $\mathcal{N}$ that has a Mahalanobis distance smaller or equal to $k\in\mathbb{R}_+$ as $\mathrm{Prob}_{k}\in[0, 1]$ and the part of the state space of $\mathbb{R}^{n_m}$ that has a Mahalanobis distance smaller or equal to $k$ as $\mathcal{S}_{k}$.

\begin{thm}
For the CBF defined in~\eqref{eq:diff_opt_cbf}, let
\begin{align}
\label{eq:min_psi_opt}
\Tilde{\psi} =\ \argmin_{\psi}\ &\ {h(x, \psi)}\ \mathrm{subject\ to}\ \psi\in\mathcal{S}_{k}.
\end{align}
Then, the solution to the TVCBFQP with $h(x, \Tilde{\psi})$ in the TVCBF constraint also guarantees safety for all $\psi \in \mathcal{S}_{k}$.
\end{thm}
\begin{proof}
Starting from~\eqref{eq:min_psi_opt}, $\forall \psi \in \mathcal{S}_{k}$ we have
\begin{align*}
    h(x, \Tilde{\psi}) &\leq h(x, \psi) &\ \text{from~\eqref{eq:min_psi_opt}}\\
    \alpha(x, \Tilde{\psi}) - \beta &\leq \alpha(x, \psi) - \beta &\ \text{from~\eqref{eq:diff_opt_cbf}}
\end{align*}
The solution of the TVCBFQP using $h(q, \Tilde{\psi})$ ensures that $h(x, \Tilde{\psi}) = \alpha(x, \Tilde{\psi}) - \beta \geq 0$, which ensures $h(x, \psi) = \alpha(x, \psi) - \beta \geq 0$. Thus, safety is guaranteed $\forall\psi \in \mathcal{S}_{k}$.
\end{proof}

Thus, using the obstacle configuration $\tilde{\psi}$ guarantees safety with probability $\mathrm{Prob}_{k}$. Note that this probability is only a lower bound since many obstacle configurations with a Mahalanobis distance greater than $k$ will have a larger $\alpha$ value than $\alpha(x, \Tilde{\psi})$. Since the optimization problem in~\eqref{eq:min_psi_opt} is computationally expensive to solve, we provide a heuristic-based approach to estimate $\tilde\psi$.

\begin{figure}[t!]
    \centering
    \includegraphics[width=0.45\textwidth]{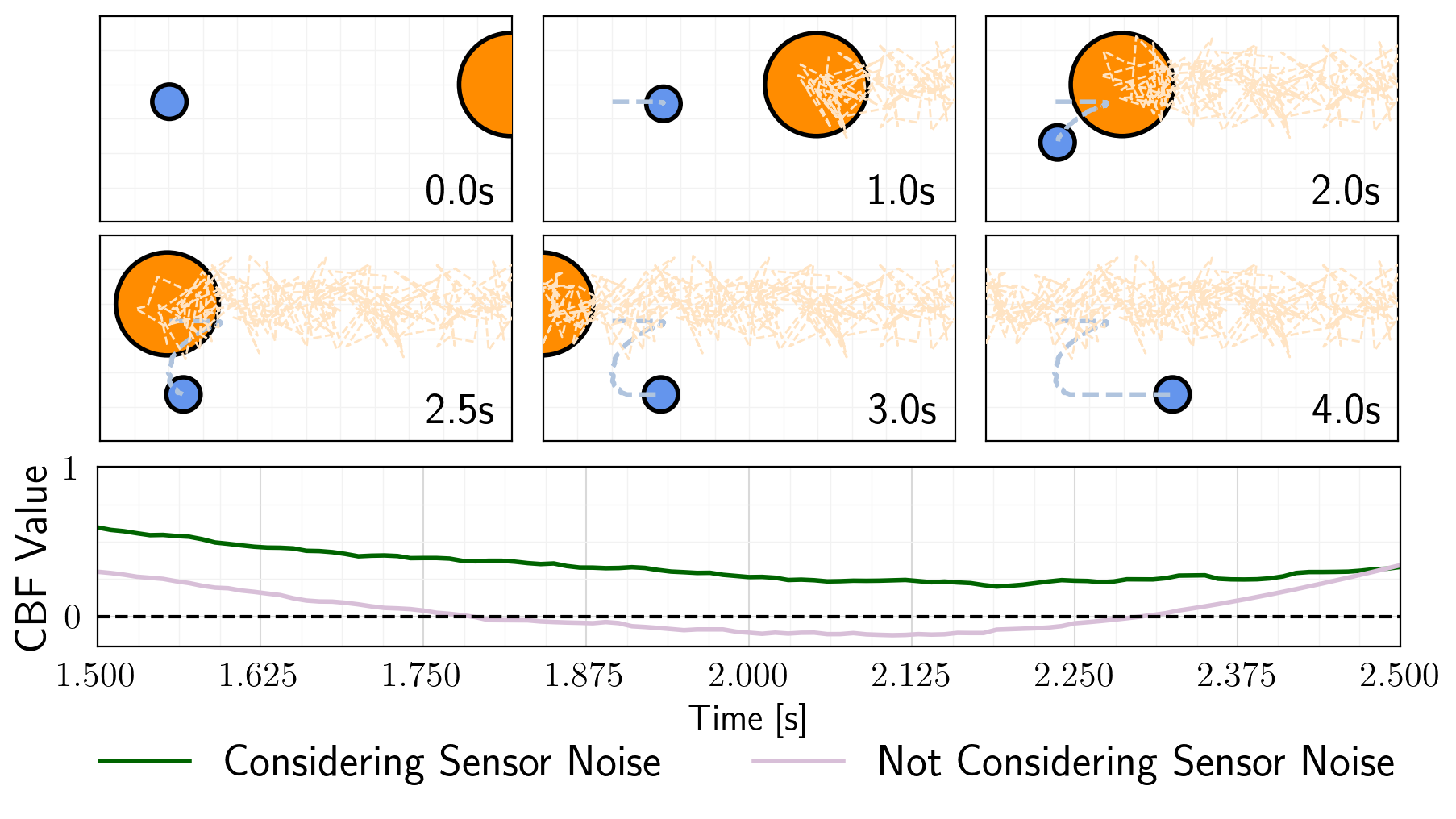}
    \caption{This figure illustrates the motion generated by the TVCBFQP controller under sensor noise. The top two rows show motion snapshots generated by the controller that considers sensor noise. For the motion snapshots, the color scheme follows Fig.~\ref{fig:moving_circle_visualization} with the only difference being the light orange dashed lines representing the estimated obstacle trajectory. The bottom plot compares the CBF value for two controllers, one considering sensor noise using the proposed method and the other not considering sensor noise. The bottom plot only shows from $[1.5, 2.5]~\si{s}$ since the CBF values are positive for the rest of the time.}
    \label{fig:moving_circle_noisy_visualization}
\end{figure}

\begin{assume}
\label{assume:orientation}
    We assume the contribution of the rotational motion to the velocity is negligible compared to the translational motion.
\end{assume}

\begin{rem}
    Assumption~\ref{assume:orientation} holds in many robotic applications, such as self-driving cars and robotic manipulation. If Assumption~\ref{assume:orientation} does not hold, our proposed method can still be applied by using a bounding shape that encapsulates the region covered when the obstacle is rotating in place.
\end{rem}

\noindent Let $\mathcal{N}(\mu_p, \Sigma_p)$ and $\mathcal{N}(\mu_q, \Sigma_q)$ denote the position distribution and quaternion distribution, respectively. Define the surface of the ellipsoid that has points with $d_M = k$ with respect to $\mathcal{N}(\mu_p, \Sigma_p)$ as $K_M$. Define $h_r$ as
\begin{equation}
    h_r = \frac{\partial h}{\partial r}\Big/\Big\|\frac{\partial h}{\partial r}\Big\| = \begin{bmatrix}
        h_{r, x} & h_{r, y} & h_{r, z}
    \end{bmatrix}^T \in \mathbb{R}^3
\end{equation}
with $r\in\mathbb{R}^3$ representing the positional terms. For a line that starts from $\mu_p$ and goes along the direction of $h_r$, we define its intersection with $K_M$ as $p_D$
\begin{equation}
\label{eq:p_danger}
    p_D = \mu_p + \frac{kh_r}{\sqrt{h_r^T\Sigma_p^{-1}h_r}} \in \mathbb{R}^3.
\end{equation}
Then, we can write the TVCBF constraint as
\begin{equation}
    \frac{\partial h}{\partial x}(x, \psi_d(t))\dot{x} + \frac{\partial h}{\partial t}(x, \psi_d(t)) \geq -\gamma h(x, \psi_d(t))
    \label{eq:measurement_noise_tvcbf}
\end{equation}
where $\psi_D(t) = (p_D(t), \mu_q(t))$. To show the usefulness of this approach in handling measurement noise, we use the Moving Circles example in Section~\ref{sec:method-tvcbf}, but with an additive Gaussian noise $\mathcal{N}(0, 0.5)$ applied to the position measurement on both the $x$ and $y$ axes and we set $\mathcal{U}$ to $\mathbb{R}^2$. This noise level is much larger than normally experienced during real robotic experiments. From Fig.~\ref{fig:moving_circle_noisy_visualization}, it can be seen that the robot maintains safety when using the proposed method to consider sensor noise and is unsafe when measurement noise is not considered. In general, the noisier the measurement, the larger $k$ should be. One can tune the $k$ value by starting with $k = 1$ and increasing or decreasing it based on empirical performance. Note that~\eqref{eq:measurement_noise_tvcbf} can be seen as moving the obstacle closer to the robot, which is different from enlarging the obstacle.

\subsection{Actuation Limits}
\label{sec:method-actuation}
\begin{figure}[t!]
    \centering
    \includegraphics[width=0.45\textwidth]{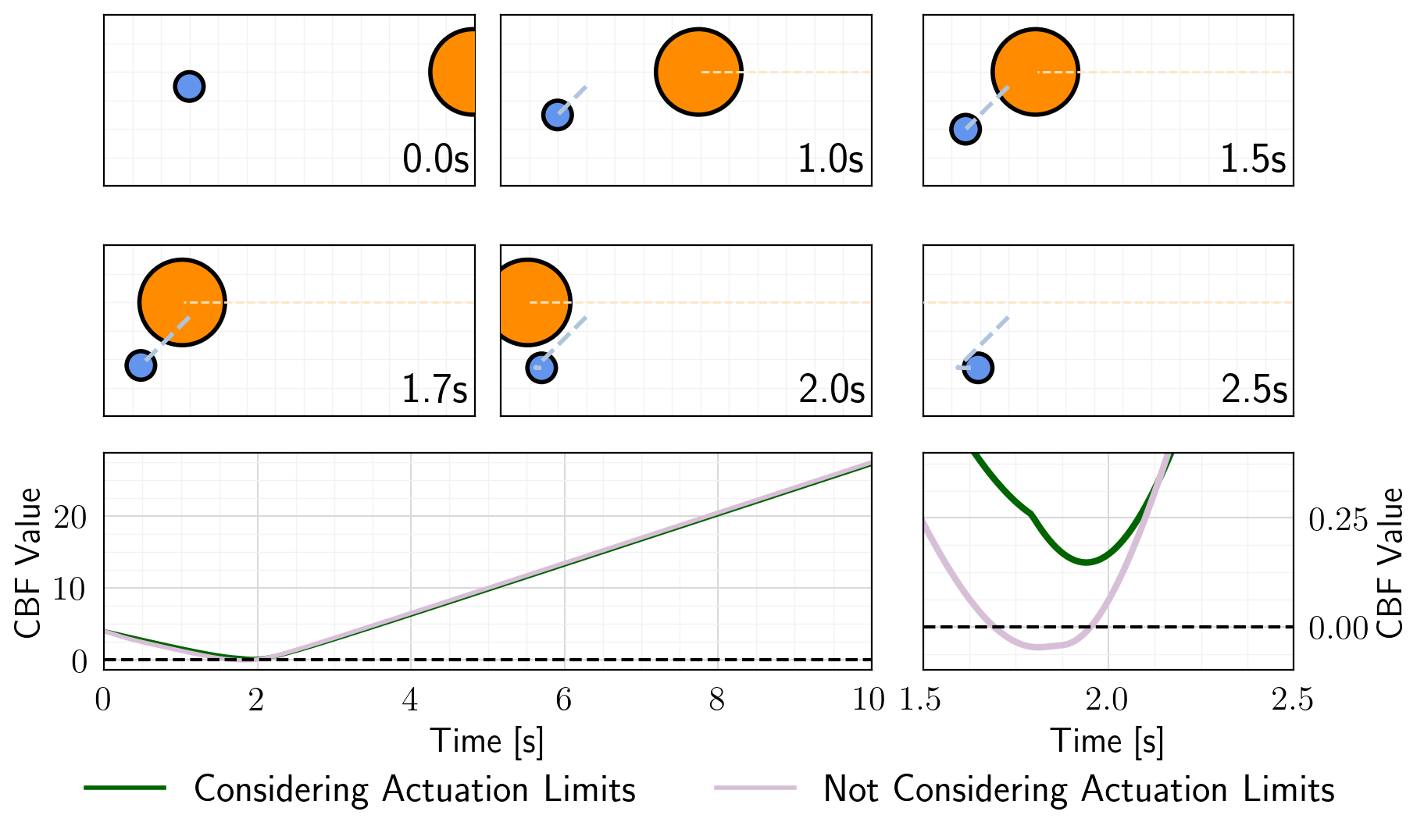}
    \caption{This figure illustrates the motion generated by the TVCBFQP controller under actuation limits. The top two rows show snapshots of the motion generated when actuation limits are considered. The color scheme follows Fig.~\ref{fig:moving_circle_visualization}. The bottom plots compare the CBF value for two controllers, one considering actuation limits using the proposed method in Section~\ref{sec:method-actuation}, and the other not considering actuation limits. The bottom right plot zooms in on the bottom left plot between time $[1.5, 2.5]$s to show that the controller not considering actuation limits leads to unsafe behavior.}
    \label{fig:moving_circle_actuation_visualization}
\end{figure}
Robotic systems cannot realize arbitrarily large joint velocity commands and have a limited tracking bandwidth. Thus, safety cannot be guaranteed if the TVCBFQP sends large joint velocity commands or requires the joint velocity to change rapidly. Under the assumption in Remark~\ref{rem:obstacle_motion}, both large and rapidly changing joint velocity commands are different consequences of the same issue: the robot did not act early enough. One can use MPC to plan a path with respect to a preview horizon to tackle the aforementioned issue. However, embedding our TVCBF constraint in an MPC will require solving~\eqref{eq:differentiable_collision} for each knot point at every iteration, which creates a nonlinear-bilevel optimization problem that cannot be solved in real-time. Alternatively, we propose to inflate the obstacle based on the relative velocity of the robot (or robot segment) and the obstacle. Define the robot velocity as $v_r = \dot{p}_r$. The obstacle's velocity computed in a frame aligned with the inertial frame but attached to the robot (or robot segment) is $\bar{v}_o = v_o - v_r$. Projecting $\bar{v}_o$ onto the unit vector pointing from the obstacle to the robot yields
\begin{equation}
    a_v = \bar{v}_o^T\frac{(p_r - p_o)}{\|(p_r - p_o)\|} = (v_o - v_r)^T\frac{(p_r - p_o)}{\|(p_r - p_o)\|} \in \mathbb{R}.
\end{equation}
For a geometric shape $B$, define the shape obtained after uniformly scaling it with a scaling factor $s$ as $B(s)$. Then, on the right-hand-side of the TVCBFQP constraint, instead of using~\eqref{eq:diff_opt_cbf} as the TVCBF, we use
\begin{figure*}[t!]
    \centering
    \includegraphics[width=0.95\textwidth]{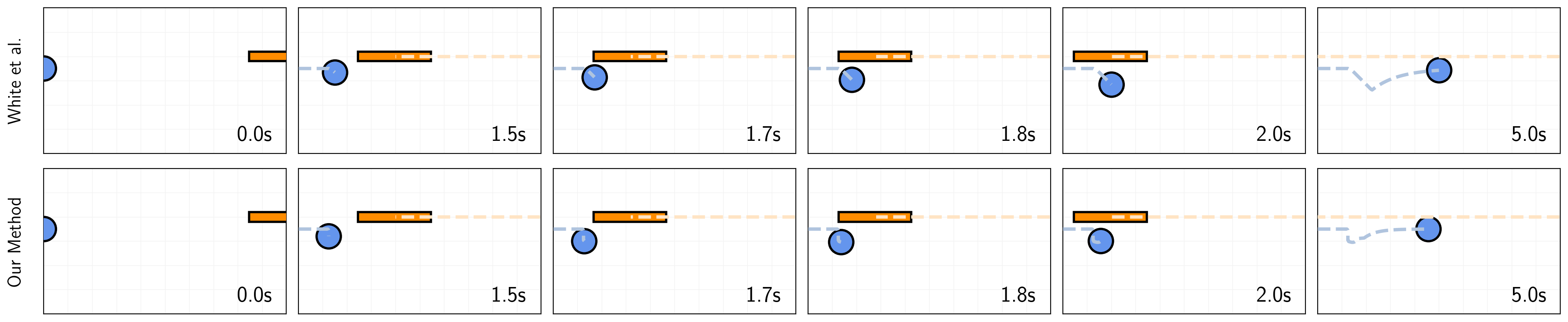}
    \caption{This figure illustrates the motion generated by our proposed controller and the controller in~\cite{WhiteJWH22} on the moving rectangle task described in Section~\ref{sec:experiments-rectangle}. The color scheme follows Fig.~\ref{fig:moving_circle_visualization}. Comparing the figures in the 3rd - 6th columns shows that the MPC-based method (top row) generates a much larger collision avoidance maneuver than our proposed approach (bottom row).}
    \label{fig:moving_rectangle_visualization}
\end{figure*}
\begin{equation}
\label{eq:expanded_tvcbf}
    h(x, \psi_B) = \alpha^\star(x, \psi_{B(s_a)}) - \beta,\ s_a = \max\{1, ba_v\}
\end{equation}
where the value of $b \in \mathbb{R}_+$ is determined by the robot's control authority and can be tuned to make the system more (make $b$ larger) or less (make $b$ smaller) conservative. The lower the control authority of the robot, the larger $b$ should be. One can start with $b = 1$ and tune its value based on empirical performances. Similar to the previous sections, we demonstrate this approach using the Moving Circles example. We use the same problem setting as in Section~\ref{sec:method-tvcbf} but set $\mathcal{U} = \{(u_x, u_y) \mid u_x, u_y \in [-1, 1]\}$ and $b = 1$. The motion generated is shown in Fig.~\ref{fig:moving_circle_actuation_visualization}. It can be seen that the proposed scaling factor inflation scheme helps ensure the safety of the robot.
\section{Experiments}
\label{sec:experiments}
In this section, we show the simulation and experimental results of utilizing our proposed method on robotic systems. First, we compare our approach with the approach proposed in~\cite{WhiteJWH22} on dynamic rectangular obstacle avoidance tasks. Then, we show the efficacy of our approach on the FR3 robot both in simulation and in real life. All experiments are performed on a PC with 32GB of RAM and an Intel Core i7 11700 processor. The optimization problems are solved using \texttt{ProxSuite}~\cite{BambadeKTC22} and \texttt{Pinocchio}~\cite{CarpentierSBMLS19} is used to compute the kinematics and dynamics terms. The two hyperparameters in our TVCBFQP controller are $\gamma$ and $\beta$. Generally, the larger the value of $\beta$ and the smaller the value of $\gamma$, the more conservative the TVCBFQP controller becomes. We set $\gamma = 1.0$ and $\beta = 1.03$ in all our experiments. The measurement noise covariance matrix is chosen based on the datasheet for the Vicon cameras and known beforehand for the simulated experiments. The process model noise covariance matrix is empirically tuned to settle a compromise between the smoothing effects and the lag it would induce. 

\begin{figure*}[t!]
    \centering
    \includegraphics[width=0.95\textwidth]{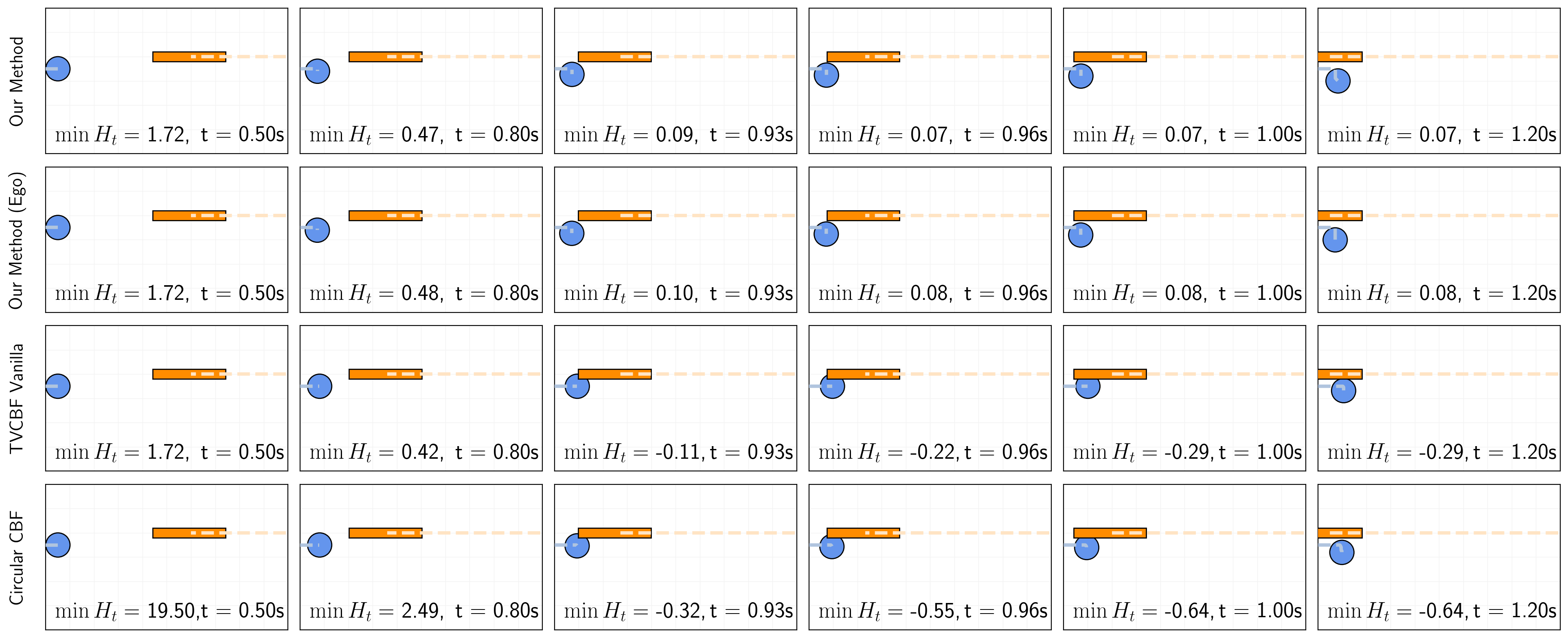}
    \caption{This figure illustrates the motion generated by our proposed controller and other CBF-based controllers. The first row shows the performance of our proposed method with world frame measurements. The second row shows the performance of our proposed method with body frame (egocentric) measurements. The third row shows the performance of the proposed diffOpt-based TVCBFQP controller without the measurement noise and actuation limit consideration. The fourth row shows the performance where the CBF is computed between the robot and three circles with a radius of 0.5 that covers the obstacle. The set $H_t$ represents the set of recorded CBF values in the period $[0, t]$, i.e., $H_t = \{h(t_0), h(t_0 + \Delta{t}), \cdots, h(t)\}$.} 
    \label{fig:moving_rectangle_visualization2}
\end{figure*}

\subsection{Dynamic Rectangular Obstacle Avoidance}
\label{sec:experiments-rectangle}
The previous Moving Circles example models both the robot and the obstacle as circles, which can be sufficiently dealt with using other methods, e.g.,~\cite{WhiteJWH22}. In this section, we replace the circular obstacle with a rectangular obstacle to show the effectiveness of our approach in handling non-ellipsoid-like obstacles. The rectangular obstacle has size  $3~\si{m} \times 0.4~\si{m} \times 2~\si{m}$. Since the motion is limited to the $xy$ plane, the size of the obstacle along the $z$ axis does not affect the generated motion. For the rectangular obstacle, the constraint in~\eqref{eq:differentiable_collision} has the form of~\cite{TracyHM22}
\begin{equation}
\label{eq:polygon_conic_constraint}
    \mathbf{A}\mathbf{R}^T r - \begin{bmatrix}
        \mathbf{A}\mathbf{R}^T & -\mathbf{b}
    \end{bmatrix}\begin{bmatrix}
        p & \alpha
    \end{bmatrix}^T \in \mathbb{R}_+
\end{equation}
where $\mathbf{A} \in \mathbb{R}^{6\times3}$ and $\mathbf{b} \in \mathbb{R}^6$ represents the halfspace constraints and $\mathbf{R} \in \mathrm{SO}(3)$ represents the orientation of the rectangular obstacle using a rotation matrix. The robot is modeled as a sphere with a radius $0.5~\si{m}$. The obstacle starts from position $[5.0, 0.0, 0.0]^T~\si{m}$ and moves with a fixed velocity of $[-4.0, 0.0, 0.0]^T~\si{m/s}$. The robot's task is to start from position $(-5.0, -0.5, 0.0)~\si{m}$ and reach the target point located at $(20.0, -0.5, 0.0)~\si{m}$ while avoiding collision with the rectangular obstacle. The robot dynamics is the same as~\eqref{eq:circle_robot_dyn}. The robot velocity along the $x$ and $y$ axis are both limited to be between $[-1, 1]~\si{m/s}$, i.e., $\mathcal{U} = \{(v_x, v_y) \mid v_x, v_y \in [-1, 1]\}$. The performance controller is a proportional controller with $K_p = 2.0$. The TVCBF with both sensor noise and actuation limits has the form of
\begin{equation}
\label{eq:combined_tvcbf}
    h(x, \psi_{D(s_a)}(t)) = \alpha^\star(x, \psi_{D(s_a)}(t)) - \beta
\end{equation}
with $\psi_{D(\cdot)}(t)$ defined as in Section~\ref{sec:measurement_noise}, $s_a$ defined as in Section~\ref{sec:method-actuation}, $k = 1.0$, $b = 1.0$, and a measurement noise of $\mathcal{N}(0.0, 0.05)$ is added along both the $x$ and $y$ measurements.

For our implementation of~\cite{WhiteJWH22}, the time horizon is $T = 1.5~\si{s}$, the sample time $\Delta{t} = 50~\si{ms}$, $d_\mathrm{risk} = d_\mathrm{obs} = 1.5~\si{m}$, $w_\mathrm{target} = 0.1$, $w_\mathrm{effort} = 0.1$, and $w_\mathrm{avoid} = 10.0$. The signed distances are computed using \texttt{hpp-fcl}~\cite{MontautLPSC22}. The motions generated by our method and~\cite{WhiteJWH22} are shown in Fig.~\ref{fig:moving_rectangle_visualization}. We see the motion generated by~\cite{WhiteJWH22} is more conservative than our method. This is due to~\cite{WhiteJWH22} modeling both the obstacle and the robot as spheres, which is different from the actual geometry and causes the generated motion to be conservative. Although $d_\mathrm{obs}$ can be tuned to generate less conservative avoidance maneuvers, this does not fundamentally solve the problem of only being able to model spherical geometries. Thus, our method's ability to model geometries using a variety of primitive shapes~\cite{TracyHM22} makes it applicable to a wider range of tasks. On average, the computation time for our proposed method is 0.29 ms, while it is 2.9~\si{ms} for the MPC-based method. Since the time complexity of linear MPCs is $O(n^3)$~\cite{BV2014}, the computation time may get prohibitively large for higher dimensional systems. Also, it would require a longer horizon if the obstacle is faster, which also increases the solution time.

Comparisons with other CBF-based methods are also performed. In Fig.~\ref{fig:moving_rectangle_visualization2}, we compare the effectiveness of our approach with three other settings: our proposed method, but the positions of the robot and the obstacle are measured in the robot's body frame; only using the diffOpt TVCBF without considering measurement noise and actuation limits; encapsulating the obstacle with three circles and compute the CBF as $h = (p_o - p_r)^T(p_o - p_r) - (R_r + R_c)^2$, with $R_c = 0.5$. From Fig.~\ref{fig:moving_rectangle_visualization2} we see that both our proposed method and the egocentric version of our proposed method can avoid collision. However, if the measurement noise and actuation limits are not considered, the motion generated by the diffOpt TVCBFQP and the circular TVCBFQP collides with the obstacle. In the four experiments shown in Fig.~\ref{fig:moving_rectangle_visualization2}, we set $v_o = [-8, 0, 0]^T \si{m/s}$, $b = 4.0$, and inherit all the other settings from the experiment shown in Fig.~\ref{fig:moving_rectangle_visualization}.

\subsection{FR3 Robotic Manipulator}
\begin{figure*}[t!]
    \centering
    \includegraphics[width=0.93\textwidth]{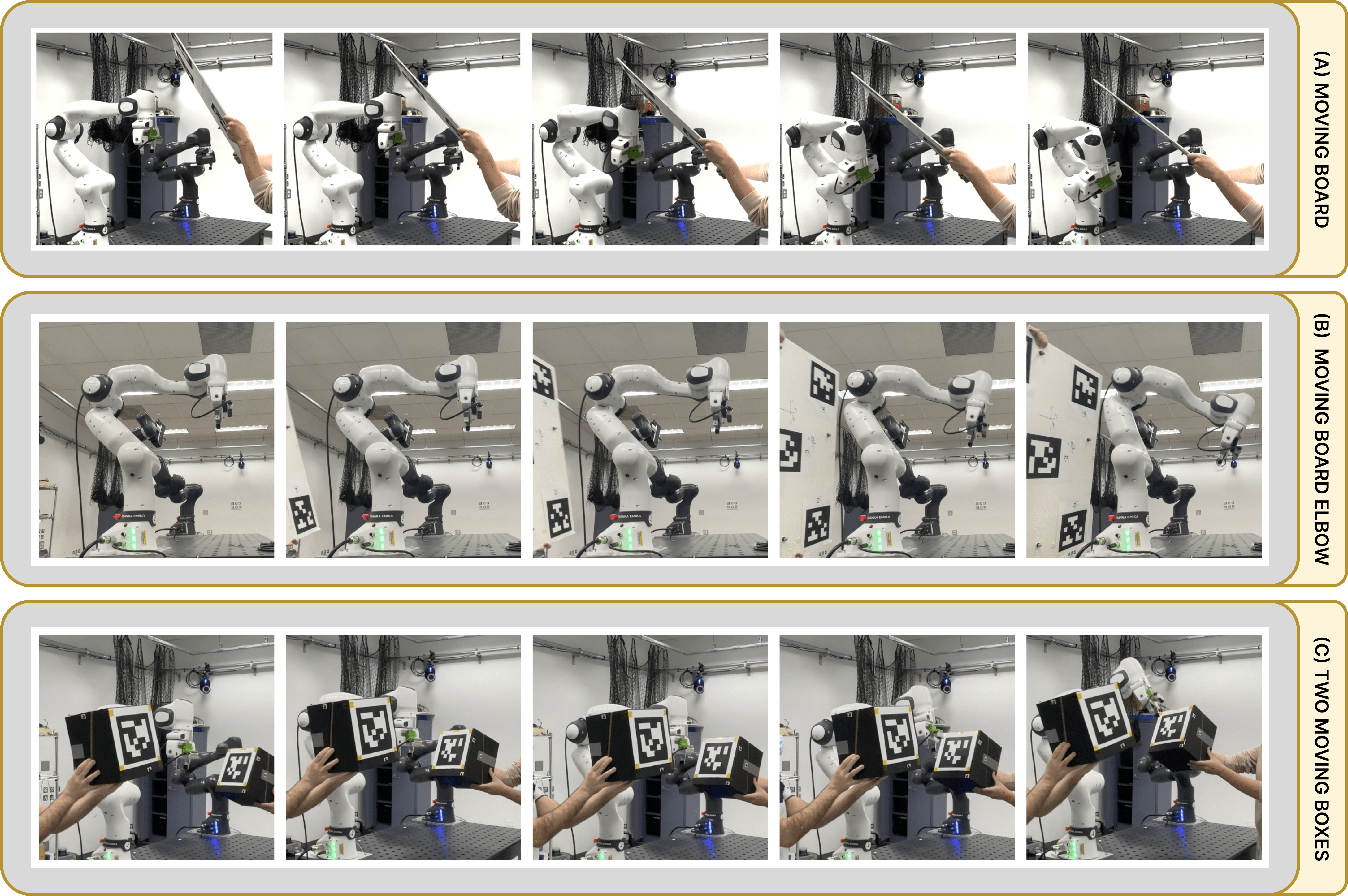}
    \caption{Snapshots of the motion generated by the TVCBFQP of the FR3 robot avoiding a moving board and two boxes. The timestamp of the snapshots increases from left to right. (A) The moving board approaches the FR3 robot manipulator from the top. (B) The moving board approaches the FR3 robot manipulator from behind. In both cases A and B, the FR3 robot retracts from the board to avoid collision. (C) The two boxes approach the FR3 robot manipulator from the bottom two sides. Since its end-effector is closer to the boxes, the FR3 robot moves the end-effector upward and away to avoid collision with the boxes. Besides these snapshots, more scenarios of the obstacles approaching the robot can be found at \url{https://youtu.be/iMmWIgvqgTU}.}
    \label{fig:snapshots}
\end{figure*}
In this section, we apply our proposed method to dynamic obstacle avoidance tasks on a 7-DOF FR3 robot manipulator. The experimental setup is shown in Fig.~\ref{fig:snapshots}. For joint velocity control, the system dynamics have the form of an integrator system, i.e., $\dot{q} = u$, with $q\in\mathbb{R}^7$ representing the joint positions. To utilize our method, we encapsulate the robot links with capsules and the end-effector with a sphere. The poses of the obstacles are measured using an array of 10 Vicon Valkyrie motion capture cameras. Then, the future position and orientation of the obstacles at the next time step are predicted using a quaternion-based EKF (Q-EKF) with a constant velocity model. The control loop runs at 100 Hz. The control limits for the FR3 robot can be found at \url{https://frankaemika.github.io/docs/}.

\subsubsection{Avoiding a Moving Board}
\label{sec:board_exp}
The task is for the robotic arm to avoid a moving board with size $0.9~\si{m}\times0.6~\si{m}\times0.02~\si{m}$. The board is modeled as a polygon and uses the same constraint as~\eqref{eq:polygon_conic_constraint}. Define a nominal joint configuration $\bar{q} \in \mathbb{R}^7$. The performance controller is a proportional-derivative (PD) controller that makes the robot return to $\bar{q}$
\begin{equation}
\label{eq:robot_pd_controller}
    u_\mathrm{ref} = K_p(\bar{q} - q) - K_d\dot{q}
\end{equation}
with $K_p, K_d \in \mathbb{R}^{7\times7}$ being diagonal matrices with the PD gains on their diagonal. The safe control is computed using~\eqref{eq:tvcbfqp}. The TVCBF constraints are computed between the links and the board. The TVCBF has the same form as~\eqref{eq:combined_tvcbf} with $k = 3.0$ and $b = 2.5$. Snapshots of the generated motion are given in Fig.~\ref{fig:snapshots}, the computed CBF values in Fig.~\ref{fig:cbf_real_robot}, and the corresponding minimum distances in Fig.~\ref{fig:exp_minimum_distance}. We see that as the board moves toward the robot with varying configurations, the proposed diffOpt TVCBFQP controller generates an avoidance maneuver that maintains a positive CBF value. On average, our proposed method takes 0.41~\si{ms} to compute the safe control.

\subsubsection{Avoiding Two Moving Boxes}
\label{sec:two_box_exp}
The task is to avoid two moving boxes with size $0.23~\si{m}\times0.25~\si{m}\times0.21~\si{m}$. The performance controller is in the same form as~\eqref{eq:robot_pd_controller}. The TVCBF is in the form of~\eqref{eq:combined_tvcbf} with $k = 3.0$ and $b = 1.0$. Snapshots of the generated motion are given in Fig.~\ref{fig:snapshots}, the CBF values with respect to the two boxes are shown in Fig.~\ref{fig:cbf_real_robot}, and the corresponding minimum distances are shown in Fig.~\ref{fig:exp_minimum_distance}. It can be seen that the proposed TVCBFQP controller can maintain safety with respect to the two moving boxes. On average, our proposed method takes 1.7~\si{ms} to compute the safe control.

\begin{rem}
    We also tested obtaining the obstacle pose using April tags~\cite{Olson11}, which yielded similar performances on the experiments in Section~\ref{sec:board_exp} and Section~\ref{sec:two_box_exp}. The main benefit of the Vicon system compared to April tags is that it enables more dynamic motion of the obstacles since the April tags would occasionally get occluded.
\end{rem}

In the two scenarios presented, since the joint measurements always contain noise, the results show that even without explicit consideration, our proposed approach is robust to such noise/error levels.

\begin{figure*}[t!]
    \centering
    \includegraphics[width=0.95\textwidth]{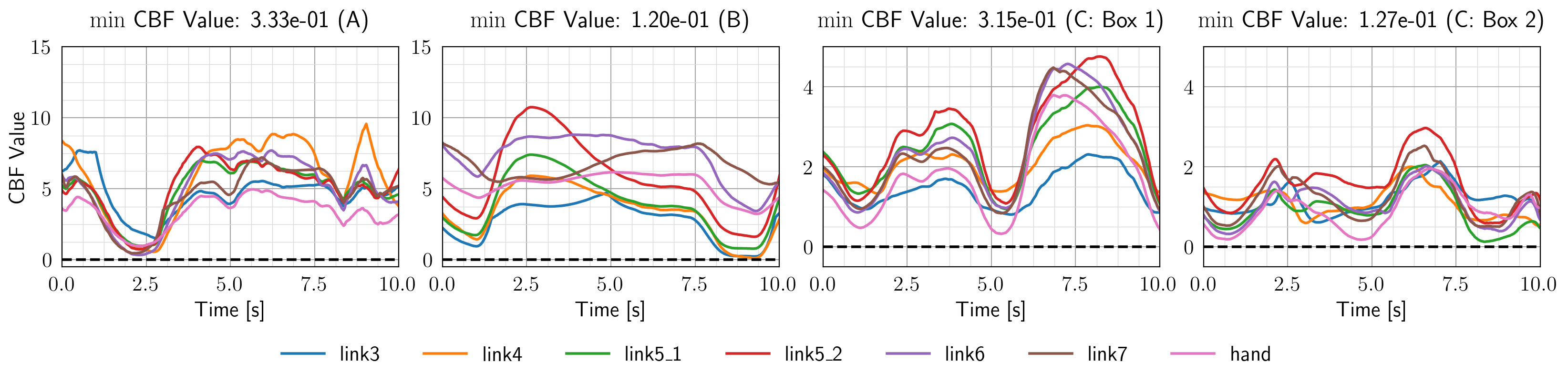}
    \vspace{-1em}
    \caption{CBF values for the robot experiments. The black dashed line represents zero CBF value. Since the CBF curves are above the dashed lines, it shows the robot avoided collision with the board/boxes. We did not include links 1 and 2 since they did not move relative to the robot's base. The A, B, and C correspond to scenarios depicted in Fig.~\ref{fig:snapshots}.}
    \label{fig:cbf_real_robot}
\end{figure*}
\begin{figure*}[t!]
    \centering
    \includegraphics[width=0.95\textwidth]{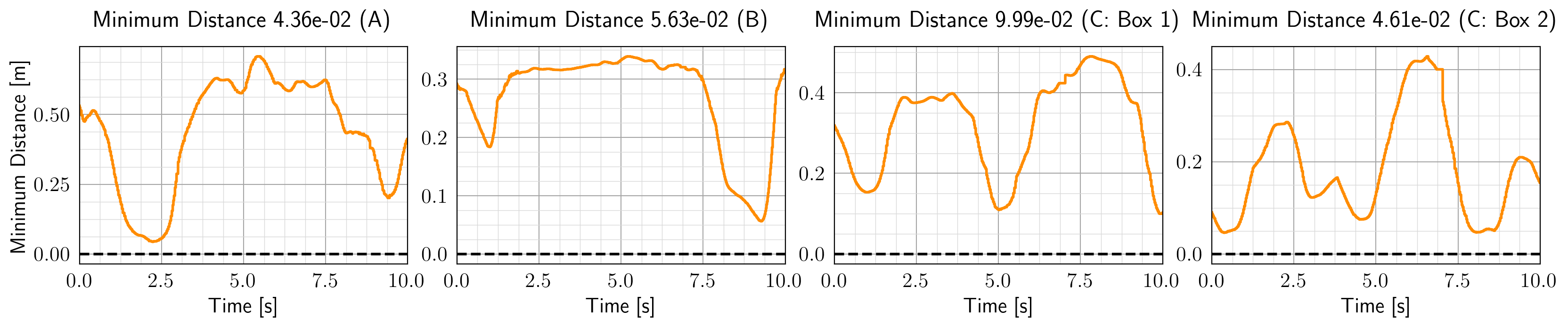}
    \caption{Minimum distance between the board/boxes and the robot. The black dashed line represents zero distance. The A, B, and C correspond to Fig.~\ref{fig:snapshots}.}
    \label{fig:exp_minimum_distance}
\end{figure*}
\section{Conclusion}
\label{sec:conclusion}

In this paper, we extended the usage of differentiable optimization based CBFs to time-varying safe sets and applied it to dynamic obstacle avoidance tasks. We proposed modifications to the TVCBFQP that make its performance robust under sensor noise and actuation limits. In a simulation environment, we have compared our method with a state-of-the-art MPC-based method and other CBF-based methods for dynamic obstacle avoidance. Given our method's ability to handle a wide range of geometric shapes when the obstacle geometry is significantly different from an ellipsoid, it was shown that our proposed method can ensure safety and generate less conservative motions. The efficacy of our approach is further tested on a 7-DOF FR3 robotic manipulator for three dynamic obstacle avoidance tasks. Besides the experimental results shown in Section~\ref{sec:experiments} more scenarios of the obstacles approaching the robot can be found at \url{https://youtu.be/iMmWIgvqgTU}. In the future, we plan to include characterizations on the type of obstacle motion our method can avoid and also explore its integration with MPC-based approaches.

\section{Acknowledgement}

We thank Pranay Gupta for his help with the experiments.
\bibliographystyle{IEEEtran}
\bibliography{IEEEabrv, refs.bib}
\end{document}